\documentclass[accepted]{uai2026} 

\usepackage[american]{babel}

\usepackage{natbib}
\usepackage{mathtools}
\usepackage{booktabs}
\usepackage{tikz}
\usepackage[utf8]{inputenc}
\usepackage{array}
\usepackage{amsmath}
\usepackage{amssymb}
\usepackage{amsthm}
\usepackage{graphicx}
\usetikzlibrary{positioning,shapes.geometric,arrows.meta,fit,calc}
\usepackage{float}
\usepackage{chngcntr}
\usepackage[ruled,vlined]{algorithm2e}
\usepackage{cleveref}
\usepackage{changepage}
\usepackage{adjustbox}

\usepackage{thmtools}
\usepackage{thm-restate}

\hypersetup{
  colorlinks=true,
  linkcolor=black,
  citecolor=black,
  urlcolor=black
}

\newcommand{\Ch}[1]{\mathrm{Ch}(#1)}
\newcommand{\ChG}[1]{\mathrm{Ch_G}(#1)}
\newcommand{\nCh}[1]{\mathrm{nCh}(#1)}
\newcommand{\Pa}[1]{\mathrm{Pa}(#1)}
\newcommand{\NA}[1]{\mathrm{NA}(#1)}
\newcommand{\EU}[1]{\mathrm{EU}[#1]}
\newcommand{\indep}{\perp \!\!\! \perp}
\DeclarePairedDelimiter\set{\lbrace}{\rbrace}

\newtheorem{definition}{Definition}
\newtheorem{theorem}{Theorem}
\newtheorem{lemma}{Lemma}

\title{A Causal Markov Condition for Value}

\author[1]{Olav Benjamin Vassend}

\affil[1]{%
    University of Inland Norway
}

\begin{document}
\maketitle

\begin{abstract}
This paper proposes a causal independence principle for \emph{value}---the \emph{value Causal Markov Condition} (v-CMC)---and develops the conceptual and mathematical foundations of a ``causal value theory'' linking causality and utility. After motivating a local formulation of the v-CMC, we introduce a probability--value duality that translates standard causal-inference results into the value setting. In particular, we formulate local, global, and decomposition versions of the v-CMC and prove their equivalence. We also define \emph{$v$-separation} and show that it is sound and complete for conditional value independence. Furthermore, we derive a Bellman-type recursion as a special case of the v-CMC, thereby generalizing standard Bellman recursion from linear chains to causal DAGs. Finally, we show how the v-CMC supports modular transfer and updating of utility information across causal contexts and develop algorithms for causally structured utility elicitation and canonical influence-diagram construction.
\end{abstract}

\section{Introduction}\label{sec:intro}

In the foundations of causal inference and decision theory, probability and value or utility occupy strikingly different positions.\footnote{We use ``value'' and ``utility'' interchangeably in this paper.} On the probability side, there is a deep theory relating causal graphs to probability distributions, whose central pillar is the Causal Markov Condition (CMC) \citep{Pearl2009, SpirtesGlymourScheines2000}. According to the CMC, any probability distribution compatible with a directed acyclic graph (DAG) encoding causal relations must respect the conditional independence relations implied by the graph.

On the value side, however, there is no analogous principle. Standard Expected Utility Theory (EUT) treats the utility function as a descriptive representation whose role is to encode an agent’s subjective preferences \citep{AnscombeAumann1963, Savage1954, vonNeumannMorgenstern1944, Jeffrey1983}. A familiar slogan is ``de gustibus non est disputandum'' (there is no arguing about taste). On this view, if an agent’s preferences exhibit complex interdependencies, then a correct value function faithfully captures all this complexity. Contrary to this standard view, the central thesis of this paper is that causal assumptions often can and should guide judgments about utility dependencies. To motivate the idea with a simple example, consider a headache pill. The utility of taking the pill plausibly depends on whether one has a headache, because one believes that the pill alleviates headaches. If one were to discover that the pill has no causal effect, one's utility of taking it would (and should) become independent of whether one has a headache.

Aside from being intuitively plausible, the idea that causal information can guide utility judgments is practically significant. The probabilistic CMC plays a central role in enabling predictions to be transferred across causal systems, since it enables modular decompositions of probability distributions that track causal structure. But decisions depend on both probabilities and utilities. Thus, extrapolating decisions from one context to another requires not only transferring probabilistic information, but also determining which aspects of a utility function should remain stable across contexts. This, in turn, calls for a way of decomposing utility functions modularly in light of causal assumptions. The aim of this paper is to formulate such a principle.

More concretely, we propose the \emph{value Causal Markov Condition} (v-CMC), which relates causal graphs to utility functions, and formulate a causal compatibility criterion for DAG--utility pairs $(G,u)$. We argue that this criterion is a natural normative constraint in many decision-making and learning settings. Furthermore, we argue that there is a conceptual duality in how causal information relates to probability and value: probability flows downstream (from causes to effects), while value flows upstream (from effects to causes). This duality allows us to formulate a translation scheme from causal-inference results to the value setting. Using it, we (i) define local, global, and decomposition versions of the v-CMC and prove they are equivalent; (ii) define $v$-separation and show it is sound and complete for conditional value independence; and (iii) derive the Bellman recursion \citep{Bellman1957} as a special case, thereby clarifying when Bellman-style decomposition holds and yielding a principled elicitation procedure over causal DAGs. Importantly, we show how the resulting decomposition identifies which local utility terms must be revised, and which can be transferred unchanged, as causal assumptions change. Finally, we propose an algorithm for automated influence diagram construction.

\subsection{Related work}
Graphs for representing preference dependencies have been studied in Multi-Attribute Utility Theory (MAUT) \citep{KeeneyRaiffa1993,Fishburn1974} and graphical extensions such as CP-nets \citep{BoutilierEtAl2004}, UCP-nets \citep{BoutilierBacchusBrafman2001}, and GAI networks \citep{BacchusGrove1995}. Related formalisms---such as Expected Utility Networks and conditional-utility-independence (CUI) networks---also study graphical representations of utility (or expected-utility) independencies \citep{LaMuraShoham1999,EngelWellman2008,LeonelliSmith2017}. Most directly, \cite{BrafmanEngel2009, BrafmanEngel2010} introduce the subtractive notion of conditional value used here and show how it supports DAG-based value decompositions. These decompositions differ from ours in three key ways. First, in the existing literature DAGs encode \emph{preference} rather than \emph{causal} structure, aiming at a modular representation of whatever dependencies an agent happens to have. The v-CMC is complementary: it specifies which value independencies an agent’s preferences \emph{should} exhibit to be compatible with causal information. Second, whereas existing (directed) graphical utility decompositions rely on parent-based ``screening-off'' criteria, the v-CMC maintains that children (in a causal graph) perform the screening-off role. Third, and crucially, the v-CMC provides a principled method for \emph{updating} utility functions modularly and locally as our causal assumptions change. Finally, although the v-CMC is complementary to, and provides a causal basis for, MAUT elicitation (as we explain in Section \ref{Applications}), its treatment and development here are intentionally framework-neutral: the v-CMC is compatible with utility frameworks other than MAUT, some of which we mention later.

A closely related tradition is influence diagrams (IDs) \citep{HowardMatheson2005, Shachter1986, Everitt2021, everitt2021b}, where diamond-shaped value nodes encode which variables contribute to overall value, often assuming an additive decomposition for convenience. Our framework provides a more principled basis for these choices in that the local v-CMC guides which value dependencies should (or should not) be represented, and the decomposition v-CMC specifies when additive decomposition of value is warranted. Indeed, in Section \ref{Applications}, we propose an algorithm that uses the v-CMC to automatically construct influence diagrams that respect causal information. In sum, the v-CMC provides a common causal framework that can be applied across Multi-Attribute Utility Theory, reinforcement learning, and influence diagrams.

\section{Setup, notation, and basic assumptions}\label{sec:basics}

Let $G=(V,E)$ be a causal DAG with vertices $V=\{X_1,\ldots,X_n\}$. We use uppercase letters such as $X_i$ for variables and lowercase letters such as $x_i$ for their realizations. When no confusion can arise, however, we will often write $X_i$ rather than $x_i$ for a particular realized value of $X_i$. For any $X_i\in V$, let $\Pa{X_i}$ denote the parents of $X_i$ in $G$, and let $\Ch{X_i}$ denote its children. Let $\NA{X_i}$ denote the set of nodes that are \emph{not} ancestors of $X_i$, excluding $X_i$ itself and its children $\Ch{X_i}$.

We assume that we have, or are interested in constructing, a utility function $u$ associated with the variables in $G$. Our treatment is framework-neutral: $u$ may be interpreted in different ways depending on the decision-theoretic framework, and nothing in the formal development below depends on choosing one such interpretation. What matters for our purposes is only the domain and scale of $u$.

Let $\mathcal{P}(V)$ be the power set of $V$, i.e., the set of all subsets of $V$. The utility function need not be defined on every subset of variables in $V$. Rather, we assume that there is some collection $S\subseteq\mathcal{P}(V)$ of variable sets for which utilities are defined. For example, $S$ might contain all subsets of $V$, or only those subsets relevant to a particular decision problem. For each $s\in S$, $u$ assigns a real number to each joint realization of the variables in $s$. We also assume that all these utilities are measured on a common scale, unique up to positive affine transformations. Thus we define:

\begin{definition}[Utility function]\label{value_function}
Let $G=(V,E)$ be a causal DAG, and let $S\subseteq\mathcal{P}(V)$ be a collection of subsets of $V$. For each $s\in S$, let $\mathcal{X}(s)$ be the set of joint realizations of the variables in $s$. A utility function over $S$ is a function

\begin{equation}
u:\ \bigcup_{s\in S}\bigl(\{s\}\times \mathcal{X}(s)\bigr)\ \to\ \mathbb{R},
\end{equation}

unique up to positive affine transformations on a common scale across all $s \in S$.
\end{definition}

The domain in Definition~\ref{value_function} may look needlessly cumbersome. The explicit $\{s\}$-tag in the domain keeps track of which variables are included in the utility assessment. For example, evaluating $X=x$ alone can be different from evaluating $X=x$ together with $Y=y$. Thus $u(s,x)$ records both the variable set $s$ being evaluated and the realization $x$ of that set. If $s\in S$ and $x\in\mathcal{X}(s)$, we write $u(x)$ rather than $u(s,x)$ whenever $s$ is clear from context. When no confusion can arise, we also write $u(X)$ for the utility of a realization of the variable or variable set $X$.

The probabilistic CMC is stated in terms of conditional probabilities. To formulate an analogous principle for value, we therefore need a notion of conditional value. Following \cite{BrafmanEngel2009, BrafmanEngel2010, Bradley2017}, we define the value of $x$ conditional on $y$ as the additional value contributed by $x$ once $y$ is already given:

\begin{equation}\label{conditional_value}
u(x \mid y) := u(x, y) - u(y).
\end{equation}

Equivalently, $u(x,y)=u(y)+u(x\mid y)$. Thus, joint value can always be written as the value of $y$ plus the conditional contribution of $x$ given $y$. Note that this definition does not assume additive separability. Additive separability is obtained only in the special case where $u(x\mid y)=u(x)$, in which case $u(x,y)=u(x)+u(y)$. Note also that the subtractive definition generalizes the familiar notion of marginal utility. If the underlying variable $y$ is continuous and $x$ is interpreted as a small increment $\Delta y$, then $u(x\mid y) = u(y+\Delta y)-u(y) \approx u'(y)\Delta y$

This definition is provided in terms of variables $x$ and $y$, but it is straightforward to extend it to sets of variables. For disjoint sets of variables $S,S'$ and realizations $s,s'$, define:

\begin{equation}\label{conditional_value2}
u(s \mid s^{\prime}) := u(s \cup s^{\prime}) - u(s^{\prime}).
\end{equation}

By convention, $u(s\mid \emptyset):=u(s)$ (so that $u(\emptyset)=0$). This yields a natural notion of conditional value independence \citep{BrafmanEngel2010}:

\begin{definition}[Conditional value independence]\label{conditional_independence}
Let $A$, $B$, and $C$ be disjoint sets of variables. If for all realizations $a,b,c$ we have $u(a \mid b,c)=u(a \mid c)$, then we write $(A \indep_u B \mid C)$, and say that $A$ is value independent of $B$ given $C$.
\end{definition}

Finally, we impose a weak consistency constraint:

\begin{definition}[Conditional consistency]\label{conditional_consistency}
Let $A$, $B$, $B'$, and $C$ be sets of variables with $B' \subseteq B$. Then
\[
(A \indep_u B \mid C) \;\Rightarrow\; (A \indep_u B' \mid C).
\]
\end{definition}

Conditional consistency says that if conditioning on $C$ renders an entire collection of variables $B$ irrelevant to the conditional value of $A$, then it should also render every subset of $B$ irrelevant. If conditional consistency were false, then the whole of $B$ could carry no value-relevant information about $A$, while some part of $B$ nevertheless does. This is arguably inconsistent with the intended meaning of screening off. Standard decision theoretic frameworks like the ones we discuss below automatically satisfy conditional consistency, but in our framework-neutral setup, we must impose it as an additional constraint. Conditional consistency ensures that conditional value independence is a semi-graphoid in the sense of \cite{PearlPaz1987} (proof in appendix):

\begin{restatable}[Conditional value independence obeys semi-graphoid axioms]{theorem}{SemiGraphoidTheorem}
\label{thm:SemiGraphoidTheorem}
Let $u$ be a value function defined as in (\ref{value_function}), let conditional value be defined as in (\ref{conditional_value2}), and let conditional value independence be defined as in Definition \ref{conditional_independence}. If conditional value independence obeys conditional consistency (Definition \ref{conditional_consistency}), then it obeys all the semi-graphoid axioms.
\end{restatable}

The discussion up to now has been rather abstract, and some assumptions may appear strong. However, several existing theories in decision theory and reinforcement learning are in fact instantiations of the above framework:

\textbf{A.} In the Jeffrey--Bolker--Bradley framework \citep{Bolker1967, Jeffrey1965, Bradley2017}, utility (or desirability) is defined on a full algebra of propositions, i.e., on the full power set $\mathcal{P}(V)$. \cite{Bradley2017} defines conditional utility as in (\ref{conditional_value}) and gives a money pump argument for why this definition is uniquely correct.

\textbf{B.} In Multi-Attribute Utility Theory (MAUT) \citep{Fishburn1974, KeeneyRaiffa1993}, one specifies local von~Neumann--Morgenstern utility functions on subsets $V' \subseteq V$ and calibrates them to a common scale; in this tradition, \cite{BrafmanEngel2010} propose the subtractive definition (\ref{conditional_value}) and the conditional independence relation (Definition \ref{conditional_independence}), and show that this relation is semi-graphoid.

\textbf{C.} In reinforcement learning (RL), the value function is defined on nested subsets $V_i$, yielding a Bellman recursion $u(V_i)=r(v_i)+u(V_{i-1})$. The advantage function $Q(a,s)-V(s)$ can be viewed as a special case of the subtractive definition (Definition \ref{conditional_value}) under the obvious identification $Q(a, s) = u(a, s)$ and $V(s) = u(s)$. See Section \ref{Applications} for more discussion of the relationship to RL.

Thus, our definition of conditional value (Definition \ref{conditional_value}) is not arbitrary, and can be justified in several different frameworks. Here we give an additional, framework-neutral justification. Since $u$ is assumed unique only up to positive affine transformations, any definition of conditional value must be affine invariant. The following result shows that, under plausible conditions, the subtractive definition is the only definition of conditional value that is affine invariant.

\begin{restatable}[Affine Invariance of Conditional Value]{theorem}{AffineInvarianceTheorem}\label{thm:AffineInvarianceTheorem}
Assume that $u(x, y)$ and $u(y)$ are utility functions that are unique up to arbitrary positive affine transformations. Assume, furthermore, that conditional utility is a function $f$ of $u(x, y)$ and $u(y)$, i.e.,  $u(x \mid y) = f(u(x, y), u(y))$, such that $f$ is differentiable at $(0, 0)$ and $f$ is increasing in $u(x, y)$ when holding fixed $u(y)$. Finally, suppose $f$ is invariant to all positive affine transformations of the utility scale, so that, for all $a > 0, b \in \mathbb{R}$:
\begin{equation}
    f(a J + b, a M + b) = a f(J, M)
\end{equation}

where $J = u(x,y)$ and $M = u(y)$. Then $f$ must be the arithmetic difference: $f(J, M) = k(J - M)$, where $k$ is a scaling factor that can be set to 1 without loss of generality.
\end{restatable}

\section{Statement and explanation of the (local) v-CMC}\label{sec:V_CMC_statement}

We are now in a position to state the local version of the v-CMC, which states that a node's causal \emph{children} screen off our evaluation of the value of the node from the node's non-ancestors:

\begin{definition}[v-CMC (local formulation)]
Let $G$ be a causal DAG with vertices $V = {X_1, \ldots, X_n}$. Let $u$ be a utility function defined on some set of subsets $S \subset \mathcal{P}(V)$. Then, for any set of non-ancestors $N \subset \NA{X_i}$:
\begin{equation*}
X_i \indep_u N \mid \Ch{X_i}
\end{equation*}
whenever $u$ is defined on the relevant subsets.
\end{definition}

Since the v-CMC only applies to subsets on which $u$ is defined, we will omit this qualification henceforth. However, all our results depend on the assumption that $u$ is defined (or can be constructed) on the relevant subsets of interest. 

The probabilistic CMC holds only for DAGs that are ``causally sufficient'' in the sense that, for any two nodes in the graph, any common cause of the nodes is also included. We need a similar condition in order to formulate our compatibility criterion:

\begin{definition}[Value sufficiency]\label{value_sufficiency}
Let $G$ be a causal DAG with vertices $V = {X_1, \ldots, X_n}$. Then $G$ is \emph{value sufficient} iff, for any two variables $X_i$ and $X_j$, all shared causal effects of $X_i$ and $X_j$ are included in $G$.
\end{definition}

Value sufficiency is the exact dual of the classic definition of causal sufficiency given by \cite{SpirtesGlymourScheines2000}. It is arguably a principle that all decision-making implicitly assumes, whether or not it relies on an explicit causal principle such as the v-CMC. For example, you might judge that the utility of drinking orange juice is independent of the medication you are taking because you assume they are causally independent. However, if you learn (e.g., from medical evidence) that orange juice and the medication causally interact to produce a nasty side effect (an omitted joint causal effect), your utilities will change, as they should.

Thus, in practice, value sufficiency is a tentative modeling assumption on which all decision making relies, and its adequacy depends on empirical and scientific facts about the underlying causal system. Typically, one will begin with a set of variables deemed relevant to the decision problem; value sufficiency then requires drawing arrows from any two variables into any shared colliders already in this set. If it turns out---on the basis of improved causal knowledge or revised modeling assumptions---that an important common effect has been omitted from the original variable set, both common sense and value sufficiency require adding it to the set and extending the utility function over the new variable. A key benefit of the v-CMC is that it enables this sort of extension to be carried out in a modular and local way, as we will see below. With these preliminaries in place, we can now formulate our causal compatibility criterion:

\begin{definition}[Compatibility criterion]\label{Compatibility}
Let $G$ be a causal DAG with vertices $V = {X_1, \ldots, X_n}$. Let $u$ be a utility function defined on some set of subsets $S \subset \mathcal{P}(V)$. If $G$ is value sufficient, then the pair $(G, u)$ is \emph{compatible} only if $u$ obeys the v-CMC with respect to the subsets of $G$ on which it is defined.
\end{definition}

Our contention is that rational decision-making should often be based on pairs $(G,u)$ that are compatible in this sense. There are multiple arguments for this claim, which will be detailed below: first, it arguably accords with reasonable judgments in concrete cases; second, it follows from widely applicable premises; third, as we will see in Section \ref{Applications}, the v-CMC and Definition \ref{Compatibility} provide a natural foundation and generalization of methods already widely applied in decision theory and reinforcement learning. We recognize that the v-CMC is likely to raise questions, and have therefore included a set of clarificatory notes in the appendix. However, two remarks are worth making already here.

First, the v-CMC, like its probabilistic dual, concerns variables that can stand in causal relationships and can therefore be meaningfully represented in a causal DAG. Many utility dependencies and independencies thus fall outside its scope. For example, it is natural to think that the disvalue of lying may depend on broader ethical considerations, such as whether Utilitarianism is true. However, an abstract proposition such as ``Utilitarianism is true'' is plausibly not a variable that can be meaningfully represented in a causal DAG. Thus, while such commitments may matter to an agent’s evaluative outlook, they fall outside the scope of a principle like the v-CMC.

Second, the v-CMC says that the children of a variable screen off its non-ancestors; it does not say that its \emph{ancestors} are screened off. A strict consequentialist would probably argue that causal ancestors should also be screened off, since once some state $B$ has occurred, its history should be irrelevant to how we assess its value, but the v-CMC does not make this demand. The v-CMC is therefore weaker than it may at first look: it does not rule out letting the causal history of a state influence how we value that state.

\section{Applications and justifications of the Compatibility criterion}\label{sec:justifications}

We begin with two simple applications of the v-CMC, to illustrate how it delivers reasonable verdicts in particular cases.

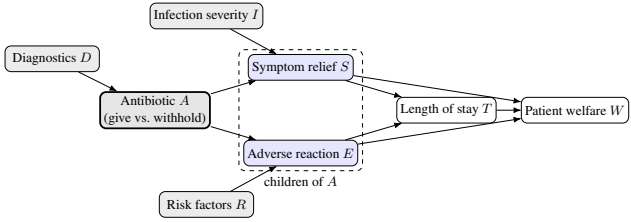
\begin{figure}[ht]
\centering
\begin{adjustbox}{width=\linewidth}
\begin{tikzpicture}[
  x=1cm,y=1cm,
  >=Latex,
  font=\small,
  node distance=10mm,
  var/.style={draw,rounded corners,align=center,minimum width=22mm,minimum height=6mm,inner sep=2.5pt},
  choice/.style={var,very thick,fill=black!10},
  child/.style={var,fill=blue!10},
  exogen/.style={var,fill=gray!15}
]

\node[choice] (A) at (0,0) {Antibiotic \(A\)\\(give vs.\ withhold)};

\node[child]  (S) at (3.4, 1.0) {Symptom relief \(S\)};
\node[child]  (E) at (3.4,-1.0) {Adverse reaction \(E\)};

\node[var]    (T) at (6.8,0)     {Length of stay \(T\)};
\node[var]    (W) at (9.8,0)     {Patient welfare \(W\)};

\node[exogen] (D) at (-2.4, 1.2) {Diagnostics \(D\)};
\node[exogen] (I) at ( 1.2, 2.2) {Infection severity \(I\)};
\node[exogen] (R) at ( 1.2,-2.2) {Risk factors \(R\)};

\draw[->] (D) -- (A);

\draw[->] (A) -- (S);
\draw[->] (A) -- (E);

\draw[->] (I) -- (S);
\draw[->] (R) -- (E);

\draw[->] (S) -- (T);
\draw[->] (E) -- (T);
\draw[->] (T) -- (W);
\draw[->] (S) -- (W);
\draw[->] (E) -- (W);

\node[draw,dashed,rounded corners,fit=(S)(E),inner sep=3pt,label=below:{\footnotesize children of \(A\)}] (childbox) {};
\end{tikzpicture}
\end{adjustbox}
\caption{Antibiotic example: $A$ affects $S,E$, which affect outcomes $T,W$; $D,I,R$ are exogenous.}
\label{figure1}
\end{figure}

Consider $A$ in Figure \ref{figure1}, and suppose the graph is value sufficient. Since $A$'s causal children are $S$ and $E$, the v-CMC implies $u(A \mid S,E,R)=u(A \mid S,E)$. That is, once we condition on symptom relief and an adverse reaction, the risk factor $R$ provides no further information relevant to assessing the utility of administering the antibiotic. This seems reasonable: if we assume the subject \emph{will} have an adverse reaction, it does not matter whether they have $R$; likewise, if we assume they \emph{will not} have an adverse reaction, $R$ again makes no difference. Hence $R$ is \emph{value irrelevant} to administering the antibiotic once we condition on $A$'s children, as entailed by the local v-CMC.

A dynamic example sharpens the point. In Figure \ref{figure1} it is reasonable that $A$ and $R$ are value dependent---e.g.\ $u(A{=}1\mid R{=}1)<u(A{=}1\mid R{=}0)$; the v-CMC permits, but does not force, such dependence. Suppose, however, that we discover $R$ is not in fact a cause of $E$, while the remainder of the DAG is correct. In this modified graph, the v-CMC entails that $A$ and $R$ are value independent. This again seems appropriate: if $R$ has no causal pathway to any variable affected by $A$, then $R$ is \emph{causally isolated} and should not influence whether to administer the antibiotic. Figure \ref{figure1} also clarifies why value sufficiency is necessary for v-CMC: if $E$ were erroneously omitted, the graph would be value insufficient since $R$ and $A$ jointly cause $E$. In that case, $R$ would still be relevant to the value of administering the antibiotic, $u(A=1\mid R=1)\neq u(A=1\mid R=0)$; but because the only causal pathway making $R$ relevant to $A$ (namely via $E$) would be omitted, the causal children of $A$ in the value insufficient graph would fail to screen off $R$ from $A$.

We now give a more general defense of the Compatibility criterion in Definition \ref{Compatibility} from three widely applicable assumptions. To motivate the first, note that the value of a node depends (in part) on its causally downstream effects: value “propagates upwards” from effects to causes, since a cause is more valuable insofar as it produces better effects. In a value-sufficient DAG, every causal path from a node $W$ to any descendant set $D'$ must pass through its children $\Ch{W}$, so $\Ch{W}$ is a bottleneck: holding $\Ch{W}$ fixed blocks every causal route by which $W$ could affect $D'$. If the value of $W$ still changes when one then learns downstream details $D'$, one is effectively treating $D'$ as an additional channel of influence not represented in the graph. Thus, if the DAG is value sufficient, it is reasonable to require:

\begin{restatable}[Child mediation]{definition}{ChildMediation}\label{def:ChildMediation}
Let $G = (V, E)$ be a causal DAG and let $u$ be a value function defined on $S \subseteq \mathcal{P}(V)$. Then $u$ is \emph{child-mediated} with respect to $G$ iff, for every node $W$ and for every set $D'$ of causal descendants of $W$ that are not children of $W$,
\begin{equation}\label{eq:child_mediation}
W \indep_u D' \mid \Ch{W}.
\end{equation}
\end{restatable}

In other words---recalling our definition of conditional utility---$u$ satisfies child mediation iff, to evaluate the conditional value contribution of $W$ over and above its causal effects $D'$, it suffices to consider the conditional value contribution of $W$ over and above its most immediate effects $\Ch{W}$. This does not mean that non-child descendants are irrelevant to total utility. Indeed, below we define a decomposition version of the v-CMC, which specifies how total utility decomposes over a causal DAG. Child mediation says only that, to determine the incremental contribution of $W$ to total value, it suffices to consider the utility of $W$ conditional on its children. This is reasonable in a value-sufficient DAG, since all of $W$'s causal effects must pass through $\Ch{W}$. 

This point is analogous to the distinction, familiar from the probabilistic CMC, between unconditional relevance and conditional screening-off. The local probabilistic CMC does not say that a node's non-descendants are irrelevant to its unconditional probability; it says that they are irrelevant to its conditional probability once the node's parents are fixed. Similarly, child mediation concerns conditional utility rather than total utility or the expected utility of an action---distal descendants can affect both of the latter even if they do not affect conditional utility.

In Figure~\ref{figure1}, for example, child mediation says that the utility of administering the antibiotic $A$, \emph{conditional} on its children $S$ and $E$, does not depend on further causal descendants such as $T$. But $T$ may still be relevant to the \emph{expected utility} of administering $A$, since the probability of $T$ may depend on $A$ through $S$ and $E$. For more on this point, see Section~\ref{sec:clarifications} of the Appendix.

The second and third assumptions draw on Causal Decision Theory (CDT) \citep{Lewis1981, Joyce1999, GibbardHarper1978, Pearl2009, Stern2017}, whose basic premise is that decisions should be based on the causal effects of the acts under consideration. Thus, the value of an intervention should not depend on causally irrelevant factors (i.e., non-descendants), a sensible requirement in many economic, political, and medical decision problems (as in Figure \ref{figure1}). 

Let $G_{do(W)}$ be obtained from $G$ by severing all arrows into $W$. Let $u_{do(W)}$ be the (possibly updated) utility function associated with $G_{do(W)}$. As we will see below (Section \ref{Applications}), $u_{do(W)}$ will sometimes differ from $u$, and the v-CMC gives a principled method for how to discover these differences. However, a basic principle consistent with CDT is that if $G'$ is a subgraph of $G$ that is not affected by the intervention, then $u_{do(W)}$ should be identical to $u$ on $G'$. More precisely, we assume: 

\begin{restatable}[Utility modularity]{definition}{Modularity}\label{def:Modularity}
If $V'\subseteq V$ is such that the induced subgraph on $V'$ is identical in $G$ and $G_{do(W)}$, i.e.
$G[V']=G_{do(W)}[V']$, then for every $S\subseteq V'$ on which both are defined, $u_{do(W)}(S)=u(S)$.
\end{restatable}

Finally, we require:

\begin{restatable}[CDT property]{definition}{CDTProperty}\label{def:CDTProperty}
Let $G = (V, E)$ be a causal DAG and let $u$ be a value function defined on some set of subsets $S \subseteq \mathcal{P}(V)$. Let $W \in V$, and let $G_{do(W)}$ and $u_{do(W)}$ be defined as above. Let $D$ be the set of all descendants of $W$ in $G_{do(W)}$ and let $ND$ be the set of non-descendants of $W$ in $G_{do(W)}$. Then $u$ has the \emph{CDT property} with respect to $G$ iff for any $ND' \subset ND$:
\begin{equation}\label{eq:cdt_property}
W \indep_{u_{do(W)}} ND' \mid D.
\end{equation}
\end{restatable}

Although the CDT property may look formally complicated, it simply requires that the utility of an intervention be assessed on the basis of its causal effects alone, with other variables relevant only insofar as they causally influence those effects. We now have:

\begin{restatable}[CDT-based justification of v-CMC]{theorem}{CDTTheorem}\label{thm:CDTTheorem}
Suppose $G$ is a DAG and $u$ is a value function on subsets of $G$ satisfying:
\begin{enumerate}
    \item $u$ has the child-mediation property with respect to $G$.
    \item $u$ obeys modularity with respect to $G$.
    \item $u$ has the CDT property with respect to $G$.
\end{enumerate}
Then $u$ obeys the local v-CMC on $G$.
\end{restatable}

This theorem provides widely applicable \emph{sufficient} conditions for the v-CMC. Child mediation is reasonable to require of any utility function over a value-sufficient DAG, and utility modularity and the CDT property are, as noted, appealing in many important decision-making contexts. However, we do not claim that the v-CMC applies \emph{only} under these conditions; a full discussion of its scope is beyond this paper.\footnote{See Section \ref{sec:clarifications} of the Appendix for more discussion of the scope of the v-CMC.}

\section{The causal probability-value duality principle}\label{DualityPrinciple}

Our discussion in the preceding section highlighted a striking asymmetry: in a causal graph, probability ``flows downwards'' whereas value ``flows upwards.'' In particular, the local v-CMC is the exact dual of the probabilistic CMC: whereas the CMC says that causal \emph{parents} probabilistically screen off variables from their \emph{non-descendants}, the v-CMC says that \emph{children} screen off (in the value sense) variables from their \emph{non-ancestors}. Taking this duality seriously suggests a systematic recipe for translating principles about DAGs and probability into corresponding principles about DAGs and value.

\begin{table}[ht]
\centering
\caption{The Causal Probability-Value Translation Key}
\label{tab:translation_key}
\renewcommand{\arraystretch}{1.3}
\begin{tabular}{@{}ll@{}}
\toprule
\textbf{Probability concept} & \textbf{Value concept} \\ \midrule
Probabilistic independence & Value independence \\
Parents ($\Pa{\cdot}$) & Children ($\Ch{\cdot}$) \\
Non-descendants & Non-ancestors \\
Common cause & Common effect \\
Parent-closed & Child-closed \\
$p(A \mid B) = \frac{p(A, B)}{p(B)}$ & $u(A \mid B) = u(A, B) - u(B)$ \\ \bottomrule
\end{tabular}
\end{table}

Using Table \ref{tab:translation_key}, we can dualize any principle relating causal DAGs and probability into a corresponding principle relating DAGs and value. While the existence of a dual does not by itself guarantee its truth, we show in the next two sections that several such utility-theoretic duals of familiar causal-inference principles can in fact be proven.

\section{Equivalent formulations and graphical criterion}\label{EquivalentFormulations}

As is well known, the probabilistic CMC admits three equivalent formulations: local, global, and decomposition (factorization). By the translation key in Table \ref{tab:translation_key}, each has a value dual. We first dualize the factorization formulation:

\begin{definition}[v-CMC (decomposition)]\label{decompositionvCMC}
Let $G=(V,E)$ be a causal DAG. A subgraph $G'=(V',E')$ is \emph{child-closed} if $X_i\in V'$ and $X_j\in \ChG{X_i}$ implies $X_j\in V'$. We say that $G$ and $u$ jointly obey the \emph{decomposition v-CMC} if for every child-closed $G'$ with vertex set $V'=\{X_{i_1},\dots,X_{i_m}\}$,
\[
u(X_{i_1},\dots,X_{i_m}) \;=\; \sum_{k=1}^m u\bigl(X_{i_k} \mid \ChG{X_{i_k}}\bigr).
\]
\end{definition}

Readers may note that Definition \ref{decompositionvCMC} is somewhat more complex than the standard factorization version of the probabilistic CMC. This is because utility functions inherently have less structure than probability functions; this extra structure must therefore be imposed in the definition. Readers familiar with the literature on generalized additive independence (GAI) utility measures may also note that Definition \ref{decompositionvCMC} essentially yields a causally based GAI decomposition.

Next, we formulate a global version via the value dual of d-separation:

\begin{definition}[$v$-separation]
Let $X$ and $Y$ be nodes (or sets of nodes) in a DAG. We say that $X$ and $Y$ are $v$-separated by $Z$ iff every path between them is \emph{blocked} by $Z$, where a path is blocked iff: (1) it contains a chain $X\to M\to Y$ (or $Y\to M\to X$) with $M\in Z$; (2) it contains a common effect $X\to M\leftarrow Y$ with $M\in Z$; or (3) it contains a common cause $X\leftarrow M\to Y$ with neither $M$ nor any ancestor of $M$ in $Z$.
\end{definition}

For example, in Figure \ref{figure1}, the only path between $A$ and $R$ goes through the common effect $E$. Hence, $E$ $v$-separates $R$ from $A$.

By construction, $v$-separation is the exact dual of $d$-separation. Hence, if $G^{\text{Rev}}$ is identical to $G$ except that all the arrows are reversed, then $X$ and $Y$ are $v$-separated by $Z$ on $G$ if and only if $X$ and $Y$ are $d$-separated by $Z$ on $G^{\text{Rev}}$---this fact will be useful below.

We are now in a position to formulate the global formulation of the v-CMC:

\begin{definition}[v-CMC (global version)]\label{globalvCMC}
Let $G$ be a causal DAG and $u$ a value function. We say that $G$ and $u$ jointly obey the \emph{global v-CMC} if whenever $X$ and $Y$ are $v$-separated by $Z$, then $X \indep_u Y \mid Z$.
\end{definition}

In the appendix, we show that all three formulations of the v-CMC are equivalent:

\begin{restatable}[v-CMC Equivalence Theorem]{theorem}{EquivalenceTheorem}\label{thm:EquivalenceTheorem}
Let $G$ be a DAG and $u$ a value function. Then:
\begin{enumerate}
    \item Local v-CMC $\iff$ Global v-CMC.
    \item Local v-CMC $\iff$ Decomposition v-CMC.
\end{enumerate}
\end{restatable}

\begin{proof}
Proof sketch: (i) \textbf{Local v-CMC} $\iff$ \textbf{Global v-CMC}: we first show that $v$-separation on a DAG $G$ is equivalent to $d$-separation on the DAG $G^{\text{Rev}}$ that is identical to $G$ except that all the arrows are reversed (the fact we stated earlier). We then use this equivalence together with the fact that, for a semi-graphoid independence relation, the Global and Local versions of the CMC are equivalent (see Theorem 2.44 of \cite{Lauritzen2019}) in order to conclude Local v-CMC $\iff$ Global v-CMC. 
(ii) \textbf{Local v-CMC} $\iff$ \textbf{Decomposition v-CMC}: The fact that the utility functions have less mathematical structure than probability functions (in particular, they cannot be marginalized) means we cannot easily use the relevant analogous proof from the causal inference literature. Instead, we prove the $\leftarrow$ direction directly and the $\rightarrow$ direction by induction on the number of nodes in $G$.
\end{proof}

Finally, $v$-separation provides a sound and complete graphical criterion for the class of value functions that obey v-CMC on $G$:

\begin{restatable}[Soundness and Completeness]{theorem}{SoundnessCompleteness}\label{thm:SoundnessCompleteness}
Let $G=(V,E)$ be a causal DAG and let $X,Y,Z\subset V$ be disjoint.
\begin{enumerate}
    \item \textbf{Soundness:} If $X$ and $Y$ are $v$-separated by $Z$ in $G$, then $X \indep_u Y \mid Z$ for all value functions $u$ that obey the decomposition v-CMC on $G$.
    \item \textbf{Completeness:} If $X$ and $Y$ are not $v$-separated by $Z$ in $G$, then there exists a value function $u$ that obeys the decomposition v-CMC on $G$ such that $X \not\indep_u Y \mid Z$.
\end{enumerate}
\end{restatable}

\noindent 

\begin{proof}
(i) \textbf{Completeness} (proof sketch; complete proof in the appendix): If $X$ and $Y$ are not $v$-separated in $G$, then they are not $d$-separated in $G^{\text{Rev}}$. Since $d$-separation is complete \citep{GeigerVermaPearl1990b} and the set of unfaithful distributions has measure 0 \citep{Meek1995}, there exists a strictly positive probability function $P$ satisfying the factorization property with respect to $G^{\text{Rev}}$ such that $X \not\indep_P Y \mid Z$. We then map $P$ to the value function $u = \log P$, and show that conditional value independence (with respect to $u$) is a semi-graphoid, and that $u$ obeys the decomposition property with respect to $G$. Finally, we show that $X \not\indep_u Y \mid Z$, and hence $u$ acts as a witness showing that if $X$ and $Y$ are not $v$-separated in $G$, then there exists a $u$ such that $X \not\indep_u Y \mid Z$.

(ii) \textbf{Soundness}: Suppose $X$ and $Y$ are $v$-separated by $Z$ in $G$ and let $u$ be a value function that obeys the decomposition v-CMC on $G$. By Theorem \ref{thm:EquivalenceTheorem}, $u$ obeys the global v-CMC on $G$, and therefore (by Definition \ref{globalvCMC}), $X \indep_u Y \mid Z$. 
\end{proof}

\section{Implications and applications}\label{Applications}

We now discuss several implications of the v-CMC, and present constructions it licenses. We begin by noting that the v-CMC generalizes the Bellman equation \citep{Bellman1957} to causal DAGs.

\subsection{A Bellman-type recursion on causal DAGs}

Define the \emph{height} of a node as the length of the longest directed path from that node to any leaf (a node with no children). Let $G_i:=(V_i,E_i)$ be the subgraph of $G$ consisting of all nodes of height at most $i$ (so $G_0$ contains exactly the leaf nodes). By construction, each $G_i$ is child-closed. We then obtain:

\begin{restatable}[Bellman-type decomposition]{theorem}{Bellman}\label{thm:Bellman}
Let $G=(V,E)$ be a causal DAG and let $u$ be a value function over $V$ satisfying the v-CMC. Let $P_i:=V_i\setminus V_{i-1}$. Then
\begin{equation}
    u(V_i) \;=\; \sum_{W \in P_i} u(W \mid \Ch{W}) \;+\; u(V_{i-1}).
\end{equation}
\end{restatable}

The proof is in the appendix. To see the connection with the standard Bellman equation, consider the special case of a linear order $W_0 \leftarrow W_1 \leftarrow \cdots \leftarrow W_n$ of temporally ordered states, and assume that $u(W\mid \Ch{W})$ is invariant to the values of $\Ch{W}$. Then $u(W\mid \Ch{W})$ depends only on $W$ and can be interpreted as a reward function $r(W)$, in which case Theorem \ref{thm:Bellman} reduces to the finite-horizon Bellman recursion

\[
u(V_i) = r(W_i) + u(V_{i-1}).
\]

Thus, the v-CMC provides a normative justification for Bellman recursion, while also clarifying both how it can be generalized and when the corresponding decomposition may fail: (1) if $u(W\mid \Ch{W})$ depends on $\Ch{W}$, the resulting decomposition is a strict generalization of the standard Bellman recursion; and (2) if the DAG is not value sufficient, the corresponding Bellman-style decomposition need not hold.

Note that the recursion in Theorem~\ref{thm:Bellman} is purely a recursion for the value function. It does not explicitly include a probability model or policy function, as the standard Bellman equations do, which is why we call it a ``Bellman-type'' recursion. However, extending the result is straightforward. By taking expectations with respect to a probability distribution and combining the probabilistic CMC with the v-CMC, one obtains a recursion for expected values. Since this extension introduces no essentially new ideas beyond those developed here, we omit the details.

In addition to allowing us to derive a causal version of the Bellman recursion, we note that the v-CMC also has a natural application in \emph{inverse} reinforcement learning. Inverse reinforcement learning \citep{Adams2022, Arora2021} is concerned with reconstructing the reward function of an agent based on their behavior. As \cite{Arora2021} note, this problem is ill-posed because there is always a large number of reward functions that could potentially fit with a finite set of observed behaviors. The set of potential reward functions therefore needs to be constrained in some way. The v-CMC arguably provides a very natural constraint.

\subsection{Causally structured utility elicitation}

The recursion in the preceding section yields a simple elicitation procedure for any child-closed graph, which is given in pseudocode in Algorithm \ref{alg:dag-elicitation}. The algorithm requires no new elicitation scheme---it is compatible with standard MAUT methods, for example---but it structures elicitation in a principled way around an assumed causal DAG:

\begin{algorithm}[]
\caption{Elicitation algorithm}
\label{alg:dag-elicitation}
\SetAlgoLined
\DontPrintSemicolon
\KwIn{Value-sufficient DAG $G=(V,E)$ of height $H$}
\KwOut{Value function $u$ over $V$}

\For{$i=0$ \KwTo $H$}{
    $u(V_i) \leftarrow u(V_{i-1})$\;
    \ForEach{$W \in P_i$}{
        Elicit $u(W \mid \Ch{W})$\;
        $u(V_i) \leftarrow u(V_i) + u(W \mid \Ch{W})$\;
    }
}
\end{algorithm}

Consider Figure~\ref{figure1}. On a high level, the procedure then has four steps: \textbf{(1)} We begin by eliciting or assigning utilities to the patient's welfare being high or low, $u(W=1)$ and $u(W=0)$. \textbf{(2)} Next, we elicit the value of Length of Stay $T$, conditional on its child $W$, i.e., $u(T=1 \mid W)$ and $u(T=0 \mid W)$. \textbf{(3)} We assess Symptom Relief ($S$) and Adverse Reaction ($E$). Their children are $\{T, W\}$, so we assess the utility of $S$ and $E$ given $T$ and $W$, i.e., $u(S \mid T, W)$ and $u(E \mid T, W)$. \textbf{(4)} Finally, we assess the antibiotic $A$ given its children $S, E$, i.e., $u(A \mid S, E)$.

Note that steps 1-4 yield local value functions. If each has been elicited using hypothetical lotteries, then they must be calibrated to a common scale. In an MAUT framework, this can be done by considering hypothetical cross-condition lotteries (see \cite{KeeneyRaiffa1993} for methods). 

This algorithm will generally have substantial complexity savings. Even in the simple example with binary variables (as in Figure \ref{figure1}), specifying the joint value function from scratch requires assigning values to all $2^n$ joint states, with $n$ the number of nodes. By contrast, a rough upper bound on the number of elicitations required by Algorithm \ref{alg:dag-elicitation} is $n\max_{W \in V} 2^{1 + |\Ch{W}|}$, i.e., linear in the number of nodes.

\subsection{Utility transfer across causal contexts}

As discussed in the introduction, a central motivation for the v-CMC is that transferring decisions across causal contexts requires not only transferring probabilistic information, but also determining which aspects of a utility function should remain stable when the underlying causal model changes. The causal decomposition provided by the v-CMC makes this possible by allowing an already constructed utility function to be modified locally and modularly as the causal graph changes. Causal graphs may change in several ways: for example, through interventions, node additions, node removals, edge additions, edge removals, or combinations of these. To illustrate the basic idea, we consider two especially important cases: (1) intervening on a variable, and (2) expanding a DAG by adding newly relevant variables.

\textbf{Interventions:} suppose we intervene on a node $X$ by setting $X=x$, thereby yielding the graph $G_{do(X)}$ where all arrows into $X$ have been severed. This changes the child sets only for the parents of $X$: for any $P \in \Pa{X}_G$, we have $X \notin \Ch{P}_{G_{do(X)}}$, while all other nodes retain the same children. Hence, under the decomposition v-CMC, updating the decomposed utility representation to remain compatible with $G_{do(X)}$ requires revising only the local terms for those $P$ whose child sets change---namely $u\!\left(P \mid \Ch{P}_{G_{do(X)}}\right)$ for $P \in \Pa{X}_G$; all other local terms can be retained unchanged.

\textbf{DAG expansions:} suppose we discover that we have omitted a common causal effect $E$ of variables in our variable set, so that the DAG is not value sufficient. We then expand $G$ to include $E$, forming the new DAG $G_{+E}$. Extending $u$ to $G_{+E}$ is straightforward: by the decomposition v-CMC, it suffices to (i) determine the new local term $u\!\left(E \mid \Ch{E}_{G_{+E}}\right)$, and (ii) update the local terms $u\!\left(W \mid \Ch{W}_{G_{+E}}\right)$ for those $W$ whose child sets changed---namely $W \in \Pa{E}_{G_{+E}}$.

Note that in both cases the cost of updating depends on the size of the subset of affected nodes, not the total size of the graph, which is what makes the v-CMC representation modular in the sense promised above. Section \ref{sec:example} of the Appendix gives a worked example showing how the decomposition v-CMC can be used to update a joint utility modularly after interventions.

\subsection{Automated influence diagram construction}\label{sec:id}

Influence diagrams are an effective tool for structuring decision problems and can be used to represent both informational and causal relationships. The v-CMC provides a principled way to construct influence diagrams that respect causal structure. In an influence diagram with causal semantics, an arc from a variable node (i.e., a chance or decision node) to a value node indicates that the variable directly influences the value node. The decomposition formulation of the v-CMC (Definition~\ref{decompositionvCMC}) can therefore be used to construct a canonical influence diagram from a value-sufficient DAG of variable nodes as follows. For each nonzero local term $u(X_i \mid \Ch{X_i})$, introduce a value node representing the local value contributed by $X_i$, and draw arcs into that value node from $X_i$ and from each member of $\Ch{X_i}$. Variables for which $u(X_i \mid \Ch{X_i})\equiv 0$ have no local value contribution and may be omitted from the value-node construction, even though they may still matter instrumentally through their causal effects. 

Thus, in summary, the procedure is: (i) input a value-sufficient DAG and the set of variables with nonzero local value contribution; (ii) create one value node for each such variable; and (iii) connect that value node to the variable and its causal children. For example, applying Algorithm~\ref{alg:id-construction} to the DAG in Figure~\ref{figure1}, with $S$, $E$, $T$, and $W$ identified as the variables with nonzero local value contribution, yields Figure~\ref{figure2}.

\begin{algorithm}
\caption{Influence diagram construction algorithm}
\label{alg:id-construction}
\SetAlgoLined
\DontPrintSemicolon
\KwIn{Value-sufficient DAG $G=(V,E)$; Variable set $V'$}
\KwOut{Canonical influence diagram}
\For{$w \in V'$}{
    Construct value node $n$ for $w$\;
    Draw arc from $w$ to $n$\;
    \ForEach{$c \in \Ch{w}$}{
        Draw arc from $c$ to $n$\;
    }
}
\end{algorithm}

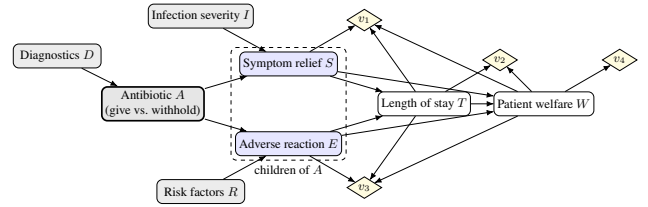
\begin{figure}[ht]
\centering
\begin{adjustbox}{width=\linewidth}
\begin{tikzpicture}[
  x=1cm,y=1cm,
  >=Latex,
  font=\small,
  node distance=10mm,
  var/.style={draw,rounded corners,align=center,minimum width=22mm,minimum height=6mm,inner sep=2.5pt},
  choice/.style={var,very thick,fill=black!10},
  child/.style={var,fill=blue!10},
  exogen/.style={var,fill=gray!15},
  value/.style={draw,diamond,aspect=1.6,align=center,inner sep=1.6pt,fill=yellow!15}
]

\node[choice] (A) at (0,0) {Antibiotic \(A\)\\(give vs.\ withhold)};

\node[child]  (S) at (3.4, 1.0) {Symptom relief \(S\)};
\node[child]  (E) at (3.4,-1.0) {Adverse reaction \(E\)};

\node[var]    (T) at (6.8,0)     {Length of stay \(T\)};
\node[var]    (W) at (9.8,0)     {Patient welfare \(W\)};

\node[exogen] (D) at (-2.4, 1.2) {Diagnostics \(D\)};
\node[exogen] (I) at ( 1.2, 2.2) {Infection severity \(I\)};
\node[exogen] (R) at ( 1.2,-2.2) {Risk factors \(R\)};

\draw[->] (D) -- (A);

\draw[->] (A) -- (S);
\draw[->] (A) -- (E);

\draw[->] (I) -- (S);
\draw[->] (R) -- (E);

\draw[->] (S) -- (T);
\draw[->] (E) -- (T);
\draw[->] (T) -- (W);
\draw[->] (S) -- (W);
\draw[->] (E) -- (W);

\node[draw,dashed,rounded corners,fit=(S)(E),inner sep=3pt,label=below:{\footnotesize children of \(A\)}] (childbox) {};

\node[value] (v1) at ($(S)+(1.9,1.1)$) {$v_1$};
\node[value] (v2) at ($(T)+(1.9,1.1)$) {$v_2$};
\node[value] (v4) at ($(W)+(1.9,1.1)$) {$v_4$};
\node[value] (v3) at ($(E)+(1.9,-1.1)$) {$v_3$};

\draw[->] (S) -- (v1);
\draw[->] (T) -- (v1);
\draw[->] (W) -- (v1);

\draw[->] (T) -- (v2);
\draw[->] (W) -- (v2);

\draw[->] (E) -- (v3);
\draw[->] (T) -- (v3);
\draw[->] (W) -- (v3);

\draw[->] (W) -- (v4);

\end{tikzpicture}
\end{adjustbox}
\caption{Antibiotic example with value nodes attached to $S$, $E$, $T$, and $W$.}
\label{figure2}
\end{figure}

Note that the value nodes are constructed to agree, in a term-by-term fashion, with the decomposition over the DAG given by the decomposition v-CMC (\ref{decompositionvCMC}). Thus, the value nodes can be added: $v_{total} = v_1 + v_2 + v_3 + v_4$.

\section{Conclusion}

This paper has developed the conceptual and mathematical foundations of how causal assumptions constrain value functions, and we have relied on assumptions it is natural to try to weaken. For example, the v-CMC relates a single assumed DAG to a value function; an important next step is to extend the framework to settings with uncertainty about the DAG, and to causal structures with cycles.

We have also explored only a small set of applications. Further directions include using the v-CMC as a causal constraint in inverse RL (as suggested earlier), developing robust ways of handling possible value insufficiency, and using $v$-separation as a diagnostic tool for checking compatibility of reported utilities with causal assumptions. Taken together, these directions suggest a broader research program in \emph{causal value theory}, in which causal structure plays a central normative role in shaping rational preferences.

\begin{acknowledgements} Thanks to Steven Diggin, Malcolm Forster, Reuben Stern, Joshua Thong, and several reviewers at UAI who made the paper better than it otherwise would have been. This paper is part of the research project ``Towards a Theory of Rational Desires'' whose broader goal is to investigate normative constraints on value functions. Work on the project is supported by funding from the European Research Council (ERC) under the European Union’s Horizon Europe research and innovation programme (Grant agreement No. 101164097).
\end{acknowledgements}

\newpage

\onecolumn

\title{A Causal Markov Condition for Value\\(Supplementary Material)}
\maketitle

\appendix

\section{A few clarificatory notes about the v-CMC}\label{sec:clarifications}

Because the v-CMC is likely to prompt several questions---as well as possibly objections---in readers' minds, we include here several (hopefully) clarifying notes about its scope. 

First, we noted earlier (in Section \ref{sec:justifications}) that the v-CMC is not committed to a causal-decision-theoretic view of decision theory, even though our justification relies on a ``CDT property.'' In fact, we do not think the plausibility of the v-CMC even depends on having a consequentialist view of value. For example, a non-consequentialist, deontological conception that regards the value of an act as determined by the motive behind the act is also compatible with the v-CMC, provided that the agent's motive can be represented as a variable in the causal model. In this case, causal non-child descendants are vacuously screened off, since \emph{all} causal consequences are regarded as irrelevant. Furthermore, non-ancestors that are not descendants are \emph{also} irrelevant, since the motive behind an act is a causal ancestor of the act. We thus think the v-CMC is compatible with both consequentialist and non-consequentialist conceptions of value, although a full explanation and defense of this claim is beyond the scope of this paper.

Second, the v-CMC constrains conditional utilities, such as $u(A \mid \Ch{A})$, not expected utilities such as $\EU{A}$. For example, in Figure \ref{figure1}, v-CMC implies that $u(A \mid S, E, I) = u(A \mid S, E)$, but $I$ obviously affects the probability of $S$, and since $S$ affects the value of $A$, it is reasonable to expect $I$ to affect $\EU{A}$, even though it is screened off in the conditional value assessment. Indeed, writing (for simplicity) the expected utility of administering the antibiotic as

\[
\EU{A \mid I} \;=\; \sum_{s,e} u(A \mid s,e)\,P(s,e \mid A,I),
\]

we see that even if $u(A \mid s,e,I)=u(A \mid s,e)$, the quantity $P(s,e \mid A,I)$ may still depend on $I$, and hence so may $\EU{A \mid I}$.

Third, even though the v-CMC implies that $A$ and $B$ are value independent if neither is a cause of the other and they have no common causal effect, this does not mean that $A$ and $B$ must be value independent \emph{conditionally} on all other variables. Consider Figure \ref{figure1} again. The utility of the patient's having an adverse reaction, given that you have administered the antibiotic to them and they experienced symptom relief is plausibly \emph{higher} than the value of the patient's having an adverse reaction given that you administered the antibiotic and they have \emph{no} symptom relief. Given that you have administered the antibiotic, having symptom relief plausibly compensates, to some degree, for having an adverse reaction. More generally, even if $X$ and $Y$ are value independent, conditioning on a common ancestor may render $X$ and $Y$ conditionally value dependent. This is analogous to how conditioning on a common effect may render $X$ and $Y$ conditionally probabilistically dependent.

\section{Proof that conditional value independence obeys semi-graphoid axioms}

First, we list the semi-graphoid axioms \citep{PearlPaz1987}:

\begin{definition}[The semi-graphoid axioms]\label{SemiGraphoidAxioms}
Let \(X,Y,W,Z\) be disjoint sets of variables.
A conditional independence relation $\indep_u$ is a semi-graphoid if it satisfies:

\begin{enumerate}
\item \textbf{Symmetry:}
\[
(X \indep_u Y \mid Z) \;\Rightarrow\; (Y \indep_u X \mid Z).
\]

\item \textbf{Semi-graphoid decomposition:}
\[
(X \indep_u Y \cup W \mid Z) \;\Rightarrow\;
(X \indep_u Y \mid Z)\ \text{and}\ (X \indep_u W \mid Z).
\]

\item \textbf{Weak Union:}
\[
(X \indep_u Y \cup W \mid Z) \;\Rightarrow\;
(X \indep_u Y \mid Z \cup W).
\]

\item \textbf{Contraction:}
\[
(X \indep_u Y \mid Z)\ \text{and}\ (X \indep_u W \mid Z \cup Y)
\;\Rightarrow\;
(X \indep_u Y \cup W \mid Z).
\]
\end{enumerate}
\end{definition}

We now prove the theorem:

\begin{theorem}[Conditional value independence obeys semi-graphoid axioms]
\label{thm:SemiGraphoidTheorem}
Let $u$ be a value function defined as in (\ref{value_function}), let conditional value be defined as in (\ref{conditional_value2}), and let conditional value independence be defined as in (\ref{conditional_independence}). If conditional value independence has the Consistency property (Definition~\ref{conditional_consistency}), then it obeys all the semi-graphoid axioms.
\end{theorem}

\begin{proof}
Throughout, let $X,Y,W,Z$ be disjoint sets of variables. Recall that, by the definition of conditional value (Definition \ref{conditional_value2}),
\[
u(X \mid Y,Z) \;=\; u(X \cup Y \cup Z) \;-\; u(Y \cup Z),
\qquad
u(X \mid Z) \;=\; u(X \cup Z) \;-\; u(Z).
\]
Hence, by Definition~\ref{conditional_independence},
\begin{equation}\label{eq:cindep_expanded}
(X \indep_u Y \mid Z)
\quad\Longleftrightarrow\quad
u(X \cup Y \cup Z) - u(Y \cup Z) \;=\; u(X \cup Z) - u(Z),
\end{equation}
for all realizations of $(X,Y,Z)$ (we suppress explicit realization notation for readability).

\paragraph{Symmetry.}
Assume $(X \indep_u Y \mid Z)$. By (\ref{eq:cindep_expanded}),
\[
u(X \cup Y \cup Z) \;=\; u(Y \cup Z) + u(X \cup Z) - u(Z).
\]
Subtracting $u(X \cup Z)$ from both sides yields
\[
u(X \cup Y \cup Z) - u(X \cup Z) \;=\; u(Y \cup Z) - u(Z),
\]
i.e.,
\[
u(Y \mid X,Z) \;=\; u(Y \mid Z).
\]
Thus $(Y \indep_u X \mid Z)$, proving Symmetry.

\paragraph{Contraction.}
Assume $(X \indep_u Y \mid Z)$ and $(X \indep_u W \mid Z \cup Y)$. Expanding each by (\ref{eq:cindep_expanded}) gives:
\begin{align}
u(X \cup Y \cup Z) - u(Y \cup Z) &= u(X \cup Z) - u(Z), \label{eq:contr1}\\
u(X \cup W \cup Y \cup Z) - u(W \cup Y \cup Z) &= u(X \cup Y \cup Z) - u(Y \cup Z). \label{eq:contr2}
\end{align}
Substituting (\ref{eq:contr1}) into the right-hand side of (\ref{eq:contr2}) yields
\[
u(X \cup W \cup Y \cup Z) - u(W \cup Y \cup Z) \;=\; u(X \cup Z) - u(Z),
\]
which, by (\ref{eq:cindep_expanded}), is exactly
\[
(X \indep_u Y \cup W \mid Z).
\]
Thus Contraction holds.

\paragraph{Semi-graphoid decomposition.}
Assume $(X \indep_u Y \cup W \mid Z)$. By the Consistency property (Definition \ref{conditional_consistency}), independence from a set implies independence from any subset of that set. Since $Y \subseteq Y \cup W$ and $W \subseteq Y \cup W$, we obtain
\[
(X \indep_u Y \mid Z)
\qquad\text{and}\qquad
(X \indep_u W \mid Z).
\]
Thus Decomposition holds.

\paragraph{Weak Union.}
Assume $(X \indep_u Y \cup W \mid Z)$. We must show $(X \indep_u Y \mid Z \cup W)$, i.e.,
\begin{equation}\label{eq:wu_goal}
u(X \cup Y \cup W \cup Z) - u(Y \cup W \cup Z)
\;=\;
u(X \cup W \cup Z) - u(W \cup Z).
\end{equation}
From $(X \indep_u Y \cup W \mid Z)$ and (\ref{eq:cindep_expanded}), we have
\begin{equation}\label{eq:wu1}
u(X \cup Y \cup W \cup Z) - u(Y \cup W \cup Z)
\;=\;
u(X \cup Z) - u(Z).
\end{equation}
By Decomposition (already established, using Consistency) applied to the same assumption $(X \indep_u Y \cup W \mid Z)$, we also have $(X \indep_u W \mid Z)$.

Expanding this by (\ref{eq:cindep_expanded}) yields

\begin{equation}\label{eq:wu2}
u(X \cup W \cup Z) - u(W \cup Z)
\;=\;
u(X \cup Z) - u(Z).
\end{equation}

Combining (\ref{eq:wu1}) and (\ref{eq:wu2}) shows that both sides of (\ref{eq:wu_goal}) are equal to $u(X \cup Z) - u(Z)$, and hence (\ref{eq:wu_goal}) holds. Therefore $(X \indep_u Y \mid Z \cup W)$, proving Weak Union.

\medskip
Since conditional value independence satisfies Symmetry, Decomposition, Weak Union, and Contraction, it obeys all the semi-graphoid axioms.
\end{proof}

\section{Uniqueness of the Subtractive Definition}

\AffineInvarianceTheorem*

\begin{proof}

Let $a=1$. Then the condition becomes:

\[
f(J+b,M+b)=f(J,M).
\]

Let $b=-M$. Then $f(J-M,0)=f(J,M)$. Define $g(z):=f(z,0)$. Then $f(J,M)=g(J-M)$, so $f$ depends only on the difference $J-M$.

Substituting into the general condition yields

\[
g\bigl((aJ+b)-(aM+b)\bigr)=a\,g(J-M),
\]

so the $b$ terms cancel and we obtain

\[
g(a(J-M))=a\,g(J-M).
\]

Letting $z=J-M$, this is $g(az)=a g(z)$ for all $a>0$. Hence, for $z>0$ we may set $a=z$ to get
$g(z)=z\,g(1)$, and for $z<0$ writing $z=-t$ with $t>0$ gives
$g(z)=g(-t)=t\,g(-1)$.
Let $g(1)=k_{+}$ and $g(-1)=-k_{-}$ (equivalently $k_-:=-g(-1)$). Then

\[
g(z)=
\begin{cases}
k_{+}\, z & \text{if } z \ge 0,\\
k_{-}\, z & \text{if } z \le 0,
\end{cases}
\qquad\text{with } k_{+},k_{-}>0.
\]

Since $f$ is differentiable at $(0,0)$, $g(z)=f(z,0)$ is differentiable at $0$, which implies that the left and right derivatives at $0$ agree, hence $k_{+}=k_{-}=:k$. Therefore $f(J,M)=k(J-M)$, i.e.,$u(x\mid y)=k\bigl(u(x,y)-u(y)\bigr)$. Since $f$ is increasing in $u(x,y)$ (holding $u(y)$ fixed), we have $k>0$. Finally, $k$ can be absorbed into the overall utility scale, so we may set $k=1$ without loss of generality.

\end{proof}

\section{Proof of the CDT-based justification of v-CMC}

We first restate the relevant definitions:

\ChildMediation*

\CDTProperty*

Here is the theorem we wish to prove:

\CDTTheorem*

\begin{proof}

Fix an arbitrary node $W \in V$. We want to show that for every set $N \subseteq \NA{W}_G$,

\begin{equation}\label{eq:vc_mc_goal}
W \indep_u N \mid \Ch{W}
\end{equation}

Let $D$ denote the set of all descendants of $W$ in $G$. Note that intervening on $W$ only removes arrows into $W$, so it does not change which nodes are reachable \emph{from} $W$. Hence, the set of descendants of $W$ is the same in $G$ and $G_{do(W)}$; in particular, we may use the same $D$ for both graphs. 

Now take any $N \subseteq \NA{W}_G$. We first decompose $N$ into its descendant and non-descendant parts:

\[
N_D \;:=\; N \cap \bigl(D \setminus \Ch{W}\bigr),
\qquad
N_{ND} \;:=\; N \setminus D.
\]

Then $N = N_D \cup N_{ND}$, and these sets are disjoint. Furthermore, $N_D \subseteq D\setminus \Ch{W}$. Using child mediation (\ref{eq:child_mediation}) with $D\setminus \Ch{W}$ yields:

\begin{equation}\label{eq:child_med_step}
W \indep_u (D\setminus \Ch{W}) \mid \Ch{W}.
\end{equation}

Use $\indep_{u_{do(W)}}$ to denote the independence relation induced by $u_{do(W)}$. Since $N_{ND} \subseteq ND$, the CDT property (\ref{eq:cdt_property}) implies

\[
W \indep_{u_{do(W)}} N_{ND} \mid D,
\]

Note that the graphical relations between $W$, $N_{ND}$, and $D$ are not affected by the intervention $do(W)$. Hence, by our assumption that $u_{do(W)}$ and $u$ are identical on all subgraphs unaffected by $do(W)$, we have:

\begin{align}
W \indep_{u_{do(W)}} N_{ND} \mid D \iff \\
W \indep_u N_{ND} \mid D
\end{align}

And therefore, since $D = \Ch{W} \cup (D \setminus \Ch{W})$, this can be written as follows:

\begin{equation}\label{eq:cdt_step}
W \indep_u N_{ND} \mid \Ch{W}, (D \setminus \Ch{W})
\end{equation}

Combining (\ref{eq:child_med_step}) and (\ref{eq:cdt_step}) using Contraction gives:

\[
W \indep_u N_{ND} \cup (D\setminus \Ch{W}) \mid \Ch{W}.
\]

Finally, using conditional consistency (\ref{conditional_consistency}), since $N_D \subseteq D\setminus \Ch{W}$, this implies:

\begin{equation}\label{eq:cond_step}
W \indep_u N_{ND} \cup N_D  \mid \Ch{W}.
\end{equation}

which is exactly (\ref{eq:vc_mc_goal}). Since $W$ and $N \subseteq \NA{W}_G$ were arbitrary, $u$ obeys the local v-CMC on $G$. 

\end{proof}

\section{Proofs of CMC equivalences}

Let us first restate the theorem:

\EquivalenceTheorem*

To show that the global and local versions of the v-CMC are equivalent, we need several results. First,  we need the semi-graphoid version of the CMC. Define:

\begin{definition}[Local SG-CMC]
     An independence relation $\indep$ and DAG $G = (V, E)$ obey the local semi-graphoid CMC (SG-CMC) if and only if, for every $W \in V$, and every set of nodes $ND$ in the set of non-descendants of $W$ (excluding $W$'s parents): 
        \begin{equation}
                W \indep ND \mid \Pa{W}
        \end{equation}
    \end{definition}

\begin{definition}[Global SG-CMC]
        An independence relation $\indep_u$ and DAG $G = (V, E)$ obey the global semi-graphoid CMC (SG-CMC) if and only if, for every disjoint set of nodes $X$, $Y$, and $Z$, if $Z$ d-separates $X$ from $Y$ in $G$, then $X \indep Y \mid Z$.
    \end{definition}

We need the following result (see Theorem 2.44 of \cite{Lauritzen2019}):

\begin{theorem}\label{semigraphoidCMC}
    Let $G=(V,E)$ be a DAG and let $\indep$ be a semi-graphoid independence relation on $V$. Then: Local SG-CMC $\iff$ Global SG-CMC
\end{theorem}

Next, we need two useful lemmas. Define $G^{\text{Rev}}$ as the DAG obtained from $G$ by reversing every edge. We then have the following results:

\begin{lemma}[Local v-CMC/CMC duality]\label{LocalDuality}
Let $G$ and $G^{\text{Rev}}$ be defined as above and let $u$ be a utility function defined on subsets of $V$, then $u$ and $G$ satisfy the local v-CMC if and only if $u$ and $G^{\text{Rev}}$ obey the local SG-CMC.
\end{lemma}
\begin{proof}
    This is immediate by the definition of $G^{\text{Rev}}$ since $\Ch{W}$ in $G$ is $\Pa{W}$ in $G^{\text{Rev}}$, and the set of non-descendants of $W$ (excluding $W$'s parents) in $G^{\text{Rev}}$ is the set of non-ancestors of $W$ (excluding $W$'s children) in $G$.
\end{proof}

\begin{lemma}[$v$-separation/d-separation duality]\label{UDDuality}
Let $G$ and $G^{\text{Rev}}$ be defined as above, then: 

$A$ and $B$ are $v$-separated by $C$ in $G$ $\iff$ $A$ and $B$ are d-separated by $C$ in $G^{\text{Rev}}$.
\end{lemma}

\begin{proof}
The proof just amounts to verifying that translating the conditions for $v$-separation using the translation key in \ref{tab:translation_key} gives $d$-separation on $G^{\text{Rev}}$. By definition, $v$-separation in $G$ blocks paths based on three conditions:

\begin{itemize}
    \item Chains $A \to M \to B$ where $M \in C$.
    \item Common Effects $A \to M \leftarrow B$ where $M \in C$.
    \item Common Causes $A \leftarrow M \to B$ where neither $M$ nor its ancestors are in $C$.
\end{itemize}

In $G^{\text{Rev}}$, all edges are reversed, so:
\begin{itemize}
    \item Chains become chains $A \leftarrow M \leftarrow B$, blocked if $M \in C$.
    \item Common Effects become Common Causes $A \leftarrow M \to B$, blocked if $M \in C$.
    \item Common Causes become Common Effects $A \to M \leftarrow B$, blocked if neither $M$ nor its descendants (which were ancestors in $G$) are in $C$.
\end{itemize}

This mapping shows that $v$-separation in $G$ corresponds exactly to the definition of $d$-separation in $G^{\text{Rev}}$, as desired.
\end{proof}

We are now in a position to prove part 1 of Theorem \ref{thm:EquivalenceTheorem}. 

\noindent \textbf{Global $\leftrightarrow$ Local}

\begin{proof}

 Note that, by Lemma \ref{LocalDuality}, $\indep_u$ and $G$ jointly obey the local v-CMC if and only if $\indep_u$ and $G^{\text{Rev}}$ jointly obey the semi-graphoid version of the local CMC for DAGs. Similarly, by Lemma \ref{UDDuality} $\indep_u$ and $G$ jointly obey the global v-CMC if and only if $\indep_u$ and $G^{\text{Rev}}$ jointly obey the semi-graphoid version of the global CMC for DAGs. Hence, since the local and global versions of the CMC are equivalent by Theorem \ref{semigraphoidCMC}, it follows that the local and global versions of the v-CMC are equivalent as well.

\end{proof}

Next, we prove the second part of Theorem \ref{thm:EquivalenceTheorem}:

\noindent \textbf{Local $\rightarrow$ Decomposition}

\begin{proof}
Let $G=(V,E)$ be the DAG and order the vertices $(Y_1,\dots,Y_n)$ so that if $Y_i \to Y_j$ then $i<j$ (parents precede children). In particular, $Y_1$ is a root (it has no parents). For brevity write $Y_{>i}:=\{Y_{i+1},\dots,Y_n\}$. Let $\Ch{Y_i}$ denote the children of $Y_i$ (in $G$), and let $\nCh{Y_{>i}} := Y_{>i}\setminus \Ch{Y_i}$. Note that under this ordering, $\nCh{Y_{>i}}$ contains no ancestors of $Y_i$. Moreover, $Y_{>i}=\Ch{Y_i}\cup \nCh{Y_{>i}}$. Because $\nCh{Y_{>i}}$ contains no ancestors of $Y_i$, the local v-CMC implies that $u(Y_i \mid \Ch{Y_i}, \nCh{Y_{>i}}) = u(Y_i \mid \Ch{Y_i})$, for all $Y_i$.

We now prove the statement by induction. If $G$ has one vertex, then $u(Y_1) = u(Y_1 \mid \emptyset) = u(Y_1 \mid \Ch{Y_1})$, and hence the decomposition property holds (vacuously). Now suppose the decomposition property holds for all graphs with fewer than $n$ vertices. Note that this assumption immediately entails that every proper subgraph of $G$ has the decomposition property. Hence, it suffices to show that $G$ itself can be decomposed into a sum of local conditional value terms. Note that the vertices $\{Y_2, \ldots, Y_n\}$ and edges between these vertices form a graph $G^{'}$. Let us use $\Ch{Y_i}_{G^{'}}$ to denote the children of $Y_i$ in the graph $G^{'}$. Our ordering guarantees that $Y_1$ is not a child of any $Y_i$. Therefore, $\Ch{Y_i}_{G^{'}} = \Ch{Y_i}_{G}$ for every $Y_i$, which allows us to drop subscripts and write $\Ch{Y_i}$. 

Using the definition of conditional value (\ref{conditional_value2}), we can write:

\begin{align}
u(Y_1, \ldots, Y_n) \\
= u(Y_1 \mid Y_2, \ldots, Y_n) + u(Y_2, \ldots, Y_n) \\
= u(Y_1 \mid \Ch{Y_1}, \nCh{Y_{>1}}) + u(Y_2, \ldots, Y_n) \\
= u(Y_1 \mid \Ch{Y_1}) + u(Y_2, \ldots, Y_n) \\
= u(Y_1 \mid \Ch{Y_1}) + \sum_{i=2}^n u(Y_i \mid \Ch{Y_i}_{G^{'}}) \\
= \sum_{i=1}^n u(Y_i \mid \Ch{Y_i})
\end{align}

Where, in the penultimate step, we used v-CMC, and in the last step we used the induction hypothesis. Since $Y_i$ is just a relabeling of $X_i$ and $u$ is insensitive to arbitrary relabelings, we have:

\begin{gather}
u(X_1, \ldots, X_n) = u(Y_1, \ldots, Y_n) \\
= \sum_{i=1}^{n}u(Y_i \mid \Ch{Y_i}) =  \sum_{i=1}^{n}u(X_i \mid \Ch{X_i})
\end{gather}

This establishes that $G$ itself can be decomposed into local conditional terms, and hence completes the proof.
\end{proof}

\noindent \textbf{Decomposition $\rightarrow$ Local}

\begin{proof}

Fix a node $X_j\in V$. Let
\[
N_j := \mathrm{nAn}{X_j}_G\setminus \Ch{X_j}_G,
\]
i.e.\ the set of non-ancestors of $X_j$ in $G$ that are not children of $X_j$.
Let $T := \{X_j\}\cup N_j$ and let $S$ be the (induced) subgraph obtained by closing $T$ under children:
starting from $T$, iteratively add $\Ch{X}$ for every $X$ already included, until no new nodes are added.
By construction, $S$ is child-closed.

Moreover, $X_j$ has no ancestors in $S$: all nodes added to $T$ are descendants (via child-closure) of nodes in $T$,
so no directed path into $X_j$ can be created inside $S$. Hence $S\setminus\{X_j\}$ is also child-closed, because if some
$Y\in S\setminus\{X_j\}$ had $X_j$ as a child then $Y$ would be an ancestor of $X_j$ in $S$, a contradiction.

Since $G$ has the decomposition property, it applies to the child-closed subgraphs $S$ and $S\setminus\{X_j\}$:

\[
u(S) = \sum_{X_i\in S} u(X_i\mid \Ch{X_i}_G),\qquad
u(S\setminus\{X_j\}) = \sum_{X_i\in S\setminus\{X_j\}} u(X_i\mid \Ch{X_i}_G).
\]

Subtracting gives

\[
u(S)-u(S\setminus\{X_j\}) = u(X_j\mid \Ch{X_j}_G).
\]

By the definition of conditional value, the left-hand side equals $u(X_j\mid S\setminus\{X_j\})$, so
\[
u(X_j\mid S\setminus\{X_j\}) = u(X_j\mid \Ch{X_j}_G).
\]

Finally, because $S$ is child-closed we have $\Ch{X_j}_S=\Ch{X_j}_G$, and by construction
$S\setminus\{X_j\} = \Ch{X_j}_G\cup N_j$. Therefore

\[
u(X_j\mid \Ch{X_j}_G\cup N_j)=u(X_j\mid \Ch{X_j}_G).
\]

By the semigraphoid decomposition axiom, this implies that for every $n\subseteq N_j$,

\[
u(X_j\mid \Ch{X_j}_G\cup n)=u(X_j\mid \Ch{X_j}_G),
\]

which is exactly the local v-CMC for $X_j$.

\end{proof}

\section{Proof of completeness}

Here is the theorem we wish to prove:

\SoundnessCompleteness

\noindent 

\begin{proof}
    
Since the main document contains the proof of soundness, here we just prove completeness.

\paragraph{(i) \textbf{Completeness}.}
Suppose $X$ and $Y$ are not $v$-separated by $Z$ in $G$. Then $X$ and $Y$ are not $d$-separated by $Z$ in the arrow-reversed graph $G^{\text{Rev}}$. By the completeness of $d$-separation \citep{VermaPearl1988, GeigerVermaPearl1990}, there exists a probability distribution $P$ that factorizes with respect to $G^{\text{Rev}}$ such that

\[
X \not\indep_P Y \mid Z.
\]

Furthermore, $P$ can be chosen strictly positive:
the subset yielding unfaithful distributions has Lebesgue measure zero \citep{Meek1995},
so a strictly positive faithful distribution exists and therefore violates $X \indep Y \mid Z$.

Now define a value function $u$ on (joint) realizations of subsets of $V$ by

\begin{equation}\label{eq:u_logP_def}
u(S) \;:=\; \log P(S),
\end{equation}

for every set of variables $S \subseteq V$ and every realization $s$ of $S$ (we suppress explicit realization notation), which is well-defined since $P$ is strictly positive. With conditional value defined by subtraction (\ref{conditional_value2}), we then have, for all disjoint $S_1,S_2 \subseteq V$,

\begin{equation}\label{eq:cond_logP}
u(S_1 \mid S_2) \;=\; u(S_1 \cup S_2) - u(S_2)
\;=\; \log P(S_1 \cup S_2) - \log P(S_2)
\;=\; \log P(S_1 \mid S_2).
\end{equation}

It follows immediately that conditional value independence with respect to $u$ coincides with probabilistic conditional independence with respect to $P$: for all disjoint $S_1,S_2,S_3 \subseteq V$,

\begin{align}\label{eq:indep_equiv}
S_1 \indep_u S_2 \mid S_3
\quad\Longleftrightarrow\quad
u(S_1 \mid S_2, S_3) = u(S_1 \mid S_3) \\
\quad\Longleftrightarrow\quad
\log P(S_1 \mid S_2, S_3) = \log P(S_1 \mid S_3)
\quad\Longleftrightarrow\quad
S_1 \indep_P S_2 \mid S_3,
\end{align}

where all equalities are understood pointwise for all realizations. In particular, since $X \not\indep_P Y \mid Z$, we obtain

\[
X \not\indep_u Y \mid Z.
\]

Finally, since $P$ factorizes with respect to $G^{\text{Rev}}$, for every \emph{parent-closed} vertex set $V' \subseteq V$ in $G^{\text{Rev}}$ we have

\begin{equation}\label{eq:factorization_Grev}
P(V') \;=\; \prod_{w \in V'} P\bigl(w \mid \Pa{w}_{G^{\text{Rev}}}\bigr).
\end{equation}

Taking logs and using (\ref{eq:u_logP_def})--(\ref{eq:cond_logP}) yields

\begin{equation}\label{eq:log_factorization}
u(V') \;=\; \sum_{w \in V'} \log P\bigl(w \mid \Pa{w}_{G^{\text{Rev}}}\bigr)
\;=\; \sum_{w \in V'} u\bigl(w \mid \Pa{w}_{G^{\text{Rev}}}\bigr).
\end{equation}

Now note that $\Pa{w}_{G^{\text{Rev}}} = \Ch{w}_G$ for every node $w$, and that $V'$ is parent-closed in $G^{\text{Rev}}$ iff $V'$ is child-closed in $G$. Hence (\ref{eq:log_factorization}) is exactly the decomposition v-CMC on $G$:

\[
u(V') \;=\; \sum_{w \in V'} u\bigl(w \mid \Ch{w}_G\bigr),
\]

for every child-closed $V' \subseteq V$.

We have thus constructed a value function $u$ that obeys the decomposition v-CMC on $G$ and for which $X \not\indep_u Y \mid Z$. This completes the proof of Completeness.

\end{proof}

\section{Proof of the Bellman type decomposition}

We start by restating the theorem:

\Bellman*

\begin{proof}

Note that by the definition of conditional value (\ref{conditional_value2}), we can write:

\begin{equation}\label{eq:Bellman1}
    u(V_i) = u(P_i \cup V_{i-1}) = u(P_i \mid V_{i-1}) + u(V_{i-1}) 
\end{equation}

Note that, by construction, $V_{i-1}$ contains all the children of $P_i$, which allows us to write $V_{i-1} = \Ch{P_i} \bigcup R$, where $\Ch{P_i}$ is the set of all the children of nodes in $P_i$ and $R$ is the remainder of nodes in $V_{i-1}$ that are not in $\Ch{P_i}$. Note that, again by construction, $R$ consists of non-ancestor nodes of $P_i$. $\Ch{P_i}$ therefore $v$-separates $P_i$ from $R$ and the (global) v-CMC implies $u(P_i \mid \Ch{P_i} \bigcup R) = u(P_i \mid \Ch{P_i})$. Thus, we can write (\ref{eq:Bellman1}) as follows:

\begin{equation}\label{eq:Bellman2}
    u(V_i) = u(P_i \mid \Ch{P_i}) + u(V_{i-1}) 
\end{equation}

Next, we will show that the following decomposition holds for any set $S$ that has the property that no node in $S$ is an ancestor of any other node in $S$:

\begin{equation}\label{eq:conditional_decomposition}
    u(S \mid \Ch{S}) = \sum_{W \in S} u(W \mid \Ch{W}),
\end{equation} 

which we do by induction on the number of nodes in $S$. If $S$ has one node, (\ref{eq:conditional_decomposition}) holds vacuously. Now suppose (\ref{eq:conditional_decomposition}) holds for all sets with fewer than $n$ nodes and suppose $S$ has $n$ nodes. Pick one of the nodes $W \in S$ and write $S':= S \setminus \set{W}$. $S'$ has $n-1$ nodes and hence satisfies (\ref{eq:conditional_decomposition}). So we can write:

\begin{align}\label{eq:conditional_decomposition2}
    u(S \mid \Ch{S}) = u(\set{W} \cup S' \mid \Ch{W} \cup \Ch{S'}) \\
    = u(\set{W} \cup \Ch{W} \cup S' \cup \Ch{S'}) - u(\Ch{W} \cup \Ch{S'}) \\
    = u(\set{W} \mid \Ch{W} \cup S' \cup \Ch{S'}) + u(\Ch{W} \cup S' \cup \Ch{S'}) - u(\Ch{W} \cup \Ch{S'}) 
\end{align} 

 Since $W$ has no ancestors in $S$, the v-CMC implies that $u(\set{W} \mid \Ch{W} \cup S' \cup \Ch{S'}) = u(\set{W} \mid \Ch{W}) = u(W \mid \Ch{W})$, so (\ref{eq:conditional_decomposition2}) becomes:

 \begin{align}\label{eq:conditional_decomposition3}
    u(S \mid \Ch{S}) = u(W \mid \Ch{W}) \\ 
    + u(\Ch{W} \cup S' \cup \Ch{S'}) - u(\Ch{W} \cup \Ch{S'}) 
    \\ = u(W \mid \Ch{W}) + u(S' \mid \Ch{S'} \cup \Ch{W})
\end{align} 

Again, by assumption $W$ cannot be an ancestor of any node in $S$, so the v-CMC implies that $u(S' \mid \Ch{S'} \cup \Ch{W}) = u(S' \mid \Ch{S'})$. We can therefore write: 

 \begin{align}\label{eq:conditional_decomposition4}
    u(S \mid \Ch{S}) = u(W \mid \Ch{W}) + u(S' \mid \Ch{S'}) \\
    u(W \mid \Ch{W}) + \sum_{W \in S'}u(W \mid \Ch{W}) \\
    = \sum_{W \in S}u(W \mid \Ch{W}),
\end{align} 

which completes the inductive proof of (\ref{eq:conditional_decomposition}). Now note that, by construction, $P_i$ is a set such that no node in $P_i$ is an ancestor of any other node in $P_i$, because if $W$ is an ancestor of $W'$, then the height of $W$ will be greater than the height of $W'$ and every node in $P_i$ has height $i$. Thus, combining (\ref{eq:Bellman2}) and (\ref{eq:conditional_decomposition}) yields the desired Bellman decomposition:

\begin{equation}
      u(V_i) = u(P_i \mid \Ch{P_i}) + u(V_{i-1}) = \sum_{W \in P_i}u(W \mid \Ch{W}) + u(V_{i-1}),
\end{equation}

which completes the proof.
    
\end{proof}

\section{A worked example of $v$-separation and v-CMC decomposition under interventions}\label{sec:example}

The purpose of this section is to illustrate how $v$-separation reduces the conditioning set required for local conditional-value assessments, and how the decomposition v-CMC permits utility functions to be updated modularly after interventions.

Consider the antibiotic DAG in Figure~\ref{figure1}. Let $G^A := G_{\operatorname{do}(A)}$ denote the intervened-upon graph in which the arrow $D\to A$ is removed, and let $u^A$ be its associated utility function. For convenience, the graph is shown in Figure~\ref{figure3}.

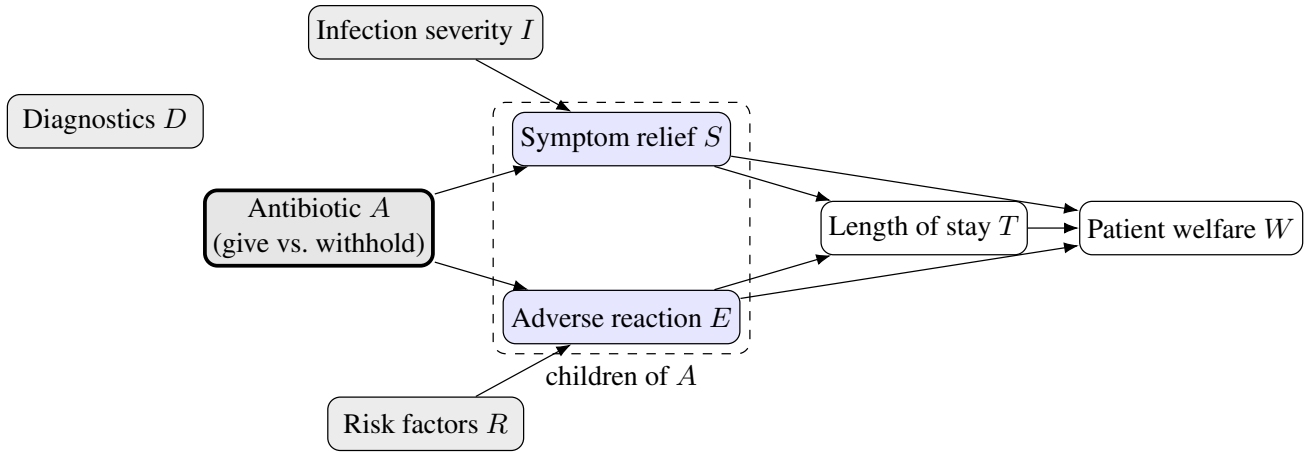
\begin{figure}[ht]
\centering
\begin{adjustbox}{width=\linewidth}
\begin{tikzpicture}[
  x=1cm,y=1cm,
  >=Latex,
  font=\small,
  node distance=10mm,
  var/.style={draw,rounded corners,align=center,minimum width=22mm,minimum height=6mm,inner sep=2.5pt},
  choice/.style={var,very thick,fill=black!10},
  child/.style={var,fill=blue!10},
  exogen/.style={var,fill=gray!15}
]

\node[choice] (A) at (0,0) {Antibiotic $A$\\(give vs.\ withhold)};

\node[child]  (S) at (3.4, 1.0) {Symptom relief $S$};
\node[child]  (E) at (3.4,-1.0) {Adverse reaction $E$};

\node[var]    (T) at (6.8,0)     {Length of stay $T$};
\node[var]    (W) at (9.8,0)     {Patient welfare $W$};

\node[exogen] (D) at (-2.4, 1.2) {Diagnostics $D$};
\node[exogen] (I) at ( 1.2, 2.2) {Infection severity $I$};
\node[exogen] (R) at ( 1.2,-2.2) {Risk factors $R$};

\draw[->] (A) -- (S);
\draw[->] (A) -- (E);

\draw[->] (I) -- (S);
\draw[->] (R) -- (E);

\draw[->] (S) -- (T);
\draw[->] (E) -- (T);
\draw[->] (T) -- (W);
\draw[->] (S) -- (W);
\draw[->] (E) -- (W);

\node[draw,dashed,rounded corners,fit=(S)(E),inner sep=3pt,label=below:{\footnotesize children of $A$}] (childbox) {};
\end{tikzpicture}
\end{adjustbox}
\caption{Post-intervention graph $G^A$: $A$ affects $S,E$, which
affect outcomes $T,W$; $D$ is disconnected and $I,R$ are exogenous.}
\label{figure3}
\end{figure}

Suppose first that we are interested in assessing the conditional value of administering the antibiotic, $A=1$, i.e., $u^A(A \mid S,E,D,I,R,T,W)$. In $G^A$, the variable $D$ is disconnected from $A$. Moreover, conditioning on $S$ and $E$ blocks the common-effect paths $A\to S\leftarrow I$ and $A\to E\leftarrow R$, and blocks every path from $A$ to $T$ or $W$. Thus, $A$ is $v$-separated from the remaining variables by its children:

\begin{equation}
A \indep {D,I,R,T,W} \mid {S,E}.
\label{eq:antibiotic-vsep-do-a}
\end{equation}

Thus, by the global v-CMC,

\begin{equation}
u^A(A \mid S,E,D,I,R,T,W) = u^A(A \mid S,E).
\label{eq:antibiotic-local-do-a}
\end{equation}

Hence, although the original graph contains eight variables, the local conditional-value assessment of administering the antibiotic requires only that we condition on the pair $(S,E)$. This illustrates how $v$-separation can simplify local conditional value assessments.

Suppose next that we want to evaluate the \emph{total} utility $u^A(A, S, E, T, W)$ on the sub-graph over the variables $\{A, S, E, T, W\}$. Note that this set is child-closed in $G^A$, and hence the decomposition v-CMC allows us to simplify as follows:

\begin{equation}
\begin{aligned}
u^A(A,S,E,T,W)
={}&u^A(W)+u^A(T\mid W) \\
&+u^A(S\mid T,W)+u^A(E\mid T,W) \\
&+u^A(A\mid S,E).
\end{aligned}
\label{eq:antibiotic-decomposition-do-a}
\end{equation}

For a numerical illustration, suppose that

\begin{equation}
\begin{aligned}
u^A(W)&=100W,
&
u^A(T\mid W)&=-10T, \\
u^A(S\mid T,W)&=10S,
&
u^A(E\mid T,W)&=-30E.
\end{aligned}
\label{eq:antibiotic-local-terms}
\end{equation}

Thus, high patient welfare has utility $100$ and low welfare has utility $0$, a long stay has disutility $10$ and a short one has a utility of $0$, and so on. Suppose, furthermore, that the local conditional value contribution of administering the antibiotic has the following slightly more complex structure:

\begin{equation}
\begin{array}{c|cccc}
(S,E) & (0,0) & (1,0) & (0,1) & (1,1) \\ \hline
u^A(A=1\mid S,E) & -8 & 5 & -15 & -5 .
\end{array}
\label{eq:antibiotic-a-term-do-a}
\end{equation}

Thus, for example, the conditional utility of administering the antibiotic is $5$ when the patient has symptom relief and no adverse reaction, whereas it is $-15$ when the patient has symptom relief and an adverse reaction.

Given these local utilities, we can calculate the total joint utility for any combination of values of $\{A,S,E,T,W\}$. For example, the total utility of administering the antibiotic when the patient has symptom relief, an adverse reaction, a short length of stay, and ultimately high welfare is:

\begin{equation}
\begin{aligned}
u^A(A=1,S=1,E=1,T=0,W=1)
&=100+0+10-30-5 \\
&=75.
\end{aligned}
\label{eq:antibiotic-example-do-a}
\end{equation}

Next, let us consider how the $v$-CMC can be used to assess how the total utility changes under a further intervention. Suppose, in particular, that we intervene to prevent adverse reactions, for example by administering a prophylactic. This removes the arrows
into $E$ and produces the new graph $G^{AE}:=G_{\operatorname{do}(A),\operatorname{do}(E=0)}$, shown in Figure~\ref{figure4}.

\begin{figure}[ht]
\centering
\begin{adjustbox}{width=\linewidth}
\begin{tikzpicture}[
  x=1cm,
  y=1cm,
  >=Latex,
  font=\small,
  node distance=10mm,
  var/.style={
    draw,
    rounded corners,
    align=center,
    minimum width=22mm,
    minimum height=6mm,
    inner sep=2.5pt
  },
  choice/.style={var,very thick,fill=black!10},
  child/.style={var,fill=blue!10},
  fixed/.style={var,fill=green!10},
  exogen/.style={var,fill=gray!15}
]

\node[choice] (A) at (0,0)
  {Antibiotic $A$\\(give vs.\ withhold)};

\node[child] (S) at (3.4,1.0) {Symptom relief $S$};
\node[fixed] (E) at (3.4,-1.0) {Adverse reaction $E=0$};

\node[var] (T) at (6.8,0) {Length of stay $T$};
\node[var] (W) at (9.8,0) {Patient welfare $W$};

\node[exogen] (D) at (-2.4,1.2) {Diagnostics $D$};
\node[exogen] (I) at (1.2,2.2) {Infection severity $I$};
\node[exogen] (R) at (1.2,-2.2) {Risk factors $R$};

\draw[->] (A) -- (S);
\draw[->] (I) -- (S);

\draw[->] (S) -- (T);
\draw[->] (E) -- (T);
\draw[->] (T) -- (W);
\draw[->] (S) -- (W);
\draw[->] (E) -- (W);

\node[
  draw,
  dashed,
  rounded corners,
  fit=(S),
  inner sep=3pt,
  label=below:{\footnotesize child of $A$}
] {};
\end{tikzpicture}
\end{adjustbox}
\caption{Post-intervention graph $G^{AE}$: $A$ affects $S$, while
$E$ is fixed at $0$ by intervention.}
\label{figure4}
\end{figure}
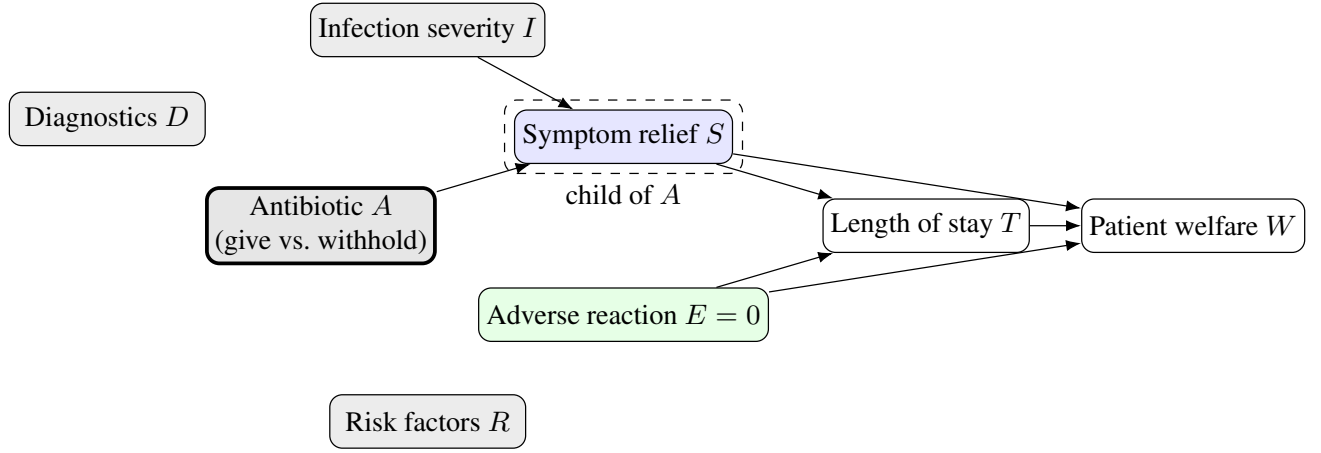

The intervention on $E$ affects the child sets of both $R$ and $A$. However, $R$ is not part of the reduced child-closed set $\{A,S,E,T,W\}$ used above, and its local term is therefore not needed for the present analysis. Since we are concerned with the downstream
utility of $A$, and $E$ is fixed at $0$ and is no longer part of the minimal child-closed set containing $A$, we may also omit $E$ from the reduced representation. The v-CMC therefore gives:

\begin{equation}
\begin{aligned}
u^{AE}(A,S,T,W)
={}&u^A(W)+u^A(T\mid W) \\
&+u^A(S\mid T,W)+u^{AE}(A\mid S).
\end{aligned}
\label{eq:antibiotic-decomposition-do-ae}
\end{equation}

Here, the only local term whose form changes in the reduced representation after intervening on $E$ is $u^{AE}(A\mid S)$; the remaining local terms in the reduced representation are retained unchanged from $u^A$. Suppose we have the following conditional values for $u^{AE}(A\mid S)$:

\begin{equation}
u^{AE}(A=1\mid S=0)=-8,
\qquad
u^{AE}(A=1\mid S=1)=2.
\label{eq:antibiotic-a-s-term}
\end{equation}

Then we can, for example, calculate the total joint utility of $(A = 1, S = 1, T = 0, W = 1)$ by using equation~\ref{eq:antibiotic-decomposition-do-ae}:

\begin{equation}
\begin{aligned}
u^{AE}(A=1,S=1,T=0,W=1)
&=100+0+10+2 \\
&=112.
\end{aligned}
\label{eq:antibiotic-example-do-ae}
\end{equation}

Of course, the numerical values are purely illustrative. The substantive point is that under $\operatorname{do}(A)$, v-separation reduces the local assessment of $A$ to $u^A(A\mid S,E)$; after $\operatorname{do}(E=0)$, it reduces further to
$u^{AE}(A\mid S)$. The remaining local terms in the reduced representation can be transferred unchanged from $u^A$.

\end{document}